\documentclass[accepted]{style} 
\usepackage[american]{babel}

\usepackage{natbib} 
\usepackage{mathtools} 
\usepackage{booktabs} 
\usepackage{tikz} 

\title{Laplace Approximation for Bayesian Tensor Network Kernel Machines}

\author[1]{\href{mailto:<A.Saiapin@tudelft.nl>?Subject=About LA-TNKM paper}{Albert Saiapin}{}}
\author[1]{Kim Batselier}
\affil[1]{%
    Delft Center for Systems and Control\\
    Delft University of Technology
}

\usepackage{booktabs}
\usepackage{graphicx}
\usepackage{caption}
\usepackage{subcaption}

\usepackage[linesnumbered, ruled, vlined]{algorithm2e}

\SetCommentSty{mycommfont}
\SetKwInput{KwInput}{Input}                
\SetKwInput{KwOutput}{Output} 
\SetKwInput{KwParams}{Parameters}  

\usepackage{amsthm}
\usepackage{thmtools} 
\theoremstyle{definition}
\newtheorem{definition}{Definition}[section]
\newtheorem{theorem}{Theorem}[section]
\newtheorem{lemma}{Lemma}[section]

\usepackage{amssymb}
\usepackage{mathabx}

\newcommand{\rBr}[1]{\left(#1\right)}
\newcommand{\sqBr}[1]{\left[#1\right]}
\newcommand{\diagOp}[1]{\operatorname{diag}\left(#1\right)}
\newcommand{\vecOp}[1]{\operatorname{vec}\left(#1\right)}
\newcommand{\tenOp}[1]{\operatorname{ten}\left(#1\right)}

\newcommand{\tens}[1]{\boldsymbol{\mathcal{#1}}} 
\newcommand{\mat}[1]{\boldsymbol{#1}} 
\newcommand{\vect}[1]{\boldsymbol{#1}} 
\newcommand{\inCo}[1]{\in\mathbb{C}^{#1}} 

\begin{document}
\maketitle

\begin{abstract}
Uncertainty estimation is essential for robust decision-making in the presence of ambiguous or out-of-distribution inputs. Gaussian Processes (GPs) are classical kernel-based models that offer principled uncertainty quantification and perform well on small- to medium-scale datasets. 
Alternatively, formulating the weight space learning problem under tensor network assumptions yields scalable tensor network kernel machines.
However, these assumptions break Gaussianity, complicating standard probabilistic inference. This raises a fundamental question: how can tensor network kernel machines provide principled uncertainty estimates?
We propose a novel Bayesian Tensor Network Kernel Machine (LA-TNKM) that employs a (linearized) Laplace approximation for Bayesian inference. A comprehensive set of numerical experiments shows that the proposed method consistently matches or surpasses Gaussian Processes and Bayesian Neural Networks (BNNs) across diverse UCI regression benchmarks, highlighting both its effectiveness and practical relevance.
\end{abstract}

\section{INTRODUCTION} 
The ability of machine learning (ML) systems to discriminate between different objects of interest provides a good value in various fields: from financial~\citep{ML_Finance_Dixon_2020} and medical~\citep{ML_Medicine_Shehab_2022} applications to natural language processing~\citep{NLP_Jang_2024} and recommender systems~\citep{RecSys_Mala_2022}. 
For example, a portable personalized medical assistant could leverage physiological measurements — such as heart rate, age, height, and weight — as input features to generate diagnostic predictions or disease risk assessments. However, in practical applications, a critical component is often overlooked: uncertainty estimation (UE)~\citep{UQ_General_Li_2012, UQ_DL_Abdar_2021}. 
One of the methods to construct uncertainty-aware models is through probabilistic modeling~\citep{PML_Intro_Murphy_2022}, where the goal is to estimate the following predictive distribution: 
\begin{equation}\label{pd_general}
    p(y^*|\vect{x}^*, \mathcal{D}),
\end{equation}
where $\mathcal{D} \coloneq \{\rBr{\vect{x}_i, y_i} | \vect{x}_i \in \mathcal{X}^D, y_i \in \mathcal{Y},  i=1, \dots, N\} = \{\mat{X}\in \mathcal{X}^{N\times D}, \vect{y}\in \mathcal{Y}^{N}\}$ is the training data and $(\vect{x}^*, y^*)$ is a test sample. 
In this case, the model produces not only mean of the distribution (predictions, e.g. cat or dog label) but also variance, reflecting its confidence. High variance indicates uncertainty, arising from limited training data, missing or noisy features, or out-of-distribution inputs. Accounting for this uncertainty enables more robust decision-making and can guide active data acquisition when the model is unsure.

A well-known example of this idea is the Gaussian Process (GP) model~\citep{GP_Rasmussen_2006}, which provides a principled probabilistic framework for kernel-based learning from a function-space perspective.
The main training step of a standard GP model consists of inverting the Gram matrix $k_{ij} \coloneq k(\vect{x}_i, \vect{x}_j)$, which encodes the pairwise relations between all training data points. As a result, the storage complexity is $\mathcal{O}(N^2)$ and the computational complexity is $\mathcal{O}(N^3)$, which limits the applicability of conventional GP models to small- or medium-scale datasets (up to $10^{5}$–$10^{6}$ observations).

On the other hand, an equivalent weight-space formulation of the GP learning problem, $f(\vect{x}) := \vect{\phi}(\vect{x})^\top\vect{w}$ with deterministic features $\vect{\phi}(\vect{x})$, allows the training complexity to be reduced from cubic to linear in the number of data samples $N$, under specific constraints. 
Following~\citep{TN_Schwab_2016, TN_Volterra_Batselier_2017, EM_Novikov_2018}, model weights $\vect{w}$ can be parameterized as a tensor network, thereby learning a multilinear (nonlinear) data-dependent representation from an exponentially large number of fixed features using only a linear number of parameters. However, as multilinearity precludes closed-form Bayesian inference, most of the previous works focus mainly on training a maximum a posteriori (MAP) parameter estimates~\citep{TD_FF_Wesel_2021, FL_Saiapin_2025}. 

In this work, we introduce the Bayesian Tensor Network Kernel Machine (LA-TNKM), which unifies variational inference and (linearized) Laplace approximation within a single framework.
The key idea is to locally approximate the posterior over model parameters with a Gaussian centered at a local maximum, where the covariance captures the local curvature. This approximation enables estimation of the predictive distribution~\eqref{pd_general} and facilitates hyperparameter evaluation via a variational inference framework. Building on this, the main contributions of this work are:
\begin{itemize}
    \item We introduce LA-TNKM, a novel probabilistic tensor network model that enables uncertainty-aware predictions while maintaining a computational cost comparable to that of a standard MAP-based tensor network kernel machine.
    \item We systematically evaluate various Hessian approximation techniques — Full, Generalized Gauss–Newton, Block-Diagonal, Diagonal, and Last Core — within the tensor network framework, highlighting their impact on memory requirements, computational efficiency, and predictive performance.
    \item We demonstrate the competitive performance of LA-TNKM on UCI regression benchmarks, where it consistently matches or outperforms Gaussian Processes, Bayesian Neural Networks, and other tensor network-based regression methods, highlighting its robustness and practical utility.
\end{itemize}

\section{BACKGROUND}
We denote sets with calligraphic capital letters - $\mathcal{Z}$.
Scalars are denoted in italics $w, W$, vectors in lowercase bold $\vect{w}$, indexed vectors as $\vect{w}_s$ or $\vect{w}_r^{(d)}$, matrices in capital bold $\mat{W}$ and tensors, also being high-order arrays, in capital italic bold font $\tens{W}$. The $i$th entry of a vector $\vect{w} \inCo{I}$ is denoted as $w_i$ and the $i_1i_2\dots i_D$th entry of a $D$th-order tensor $\tens{W} \inCo{I_1 \times I_2 \times \dots \times I_D}$ as $w_{i_1 i_2 \dots i_D}$. 
The conjugate-transpose of $\mat{A}$ is denoted as $\mat{A}^\top$ 
and $\otimes$, $\odot_R$, $\oasterisk$ represent the Kronecker product, row-wise Khatri-Rao product, Hadamard product correspondingly~\citep{TN1_Cichocki_2016}. The $\diagOp{\cdot}$ operator returns a diagonal matrix when applied to a vector. We employ zero-based indexing for all tensors.
Since working with vectors and matrices is more straightforward than with tensors, we introduce the vectorization operation.
\begin{definition}[Vectorization]\label{def_vec}
    The vectorization operator $\vecOp{\cdot}: \mathbb{C}^{I_1 \times I_2 \times \dots \times I_D} \rightarrow \mathbb{C}^{I_1I_2\dots I_D}$ is defined as:
    \begin{equation*}
        \vecOp{\tens{W}}_i \coloneq w_{i_1 i_2 \dots i_D},
    \end{equation*}
    where $i \coloneq \rBr{i_1i_2\dots i_D}= i_1 + \sum_{d=2}^{D}i_d \prod_{j=1}^{d-1}I_j$.
\end{definition}
Its inverse, the tensorization operator, is further denoted as $\tenOp{\vect{w}, I_1, I_2, \dots, I_D}$.

\subsection{Tensor Networks}
The fundamental idea underlying tensor networks (TNs)~\citep{TD_Kolda_2009, TD_Cichocki_2014} is to efficiently represent and manipulate high-dimensional tensors $\tens{W}$ by decomposing them into a network of smaller, low-rank tensors (cores) connected by shared indices. 
In this work, we focus on the canonical polyadic decomposition.
\begin{definition}[Canonical polyadic decomposition (CPD)~\citep{TD_Kolda_2009}]\label{CPD_Kron}
    A $D$th-order tensor $\tens{W} \in \mathbb{C}^{I_1 \times I_2 \times \dots \times I_D}$ has a rank-$R$ CPD if
\begin{equation}
    \vecOp{\tens{W}} \coloneq \sum_{r=1}^{R} \vect{v}^{(D)}_{r} \otimes \dots \otimes \vect{v}^{(1)}_{r},
\end{equation}  
where $\mat{V}^{(d)} \coloneq \sqBr{\vect{v}^{(d)}_{1}, \dots, \vect{v}^{(d)}_{R}} \in \mathbb{C}^{I_d \times R}$ are called the cores of the CPD.
\end{definition}
From Definition~\ref{CPD_Kron}, the storage complexity of a $D$th-order tensor can be reduced from exponential $\mathcal{O}(I^D)$ to linear $\mathcal{O}(DIR)$ that is beneficial when $D$ is large and $I \coloneq \max(I_1, \dots, I_D)$. 
An additional advantage of the CPD is its simplicity, as it is governed by a single hyperparameter — the scalar rank-$R$. In contrast, Tensor Train (TT)~\citep{TT_Oseledets_2011} and Hierarchical Tucker (HT)~\citep{HT_Grasedyck_2010} decompositions involve multiple rank parameters, making hyperparameter tuning more complex. 

\subsection{Tensor Network Kernel Machine}
Consider the following model of linear regression:
\begin{equation}\label{intro_simple_model}
    f(\vect{x}) = \vect{\phi}(\vect{x})^\top\vect{w},
\end{equation}
where $\vect{\phi}(\vect{x}) \in\mathbb{C}^{I_1 I_2 \dots I_D}$ represents nonlinear features,  $\vect{w}\in\mathbb{C}^{I_1 I_2 \dots I_D}$ is a vector of weights learned from the data $\mathcal{D}$.
The feature map $\vect{\phi}(\vect{x})$ is important as it allows for modeling various nonlinear behaviors in the data.
In addition, in the context of tensor network kernel machines two assumptions are made:
\begin{enumerate}
    \item Features $\vect{\phi}(\vect{x})$ have tensor-product structure: \begin{equation}\label{tensor_product_features}
    \vect{\phi}(\vect{x}) \coloneq \vect{\phi}^{(D)}(x_D) \otimes \dots \otimes \vect{\phi}^{(1)}(x_1),
    \end{equation}
    where $\vect{\phi}^{(d)}: \mathbb{C} \rightarrow \mathbb{C}^{I_d}$ is a feature map acting on the $d$th component of $\vect{x} \inCo{D}$.
    \item Model weights $\vect{w}$ are represented as a tensor network.
\end{enumerate}
The tensor-product structure~\eqref{tensor_product_features} is directly related to product kernels \citep{VFF_Hensman_2017, GPRR_Solin_2019}, Fourier features~\citep{TT_FF_Wahls_2014} and polynomials~\citep{KM__Cristianini_2004}. At first glance, a disadvantage of the tensor-product structure in Equation~\eqref{tensor_product_features} is that the input vector $\vect{x}$ is mapped into an exponentially large feature vector $\vect{\phi}(\vect{x}) \inCo{I_1 I_2 \dots I_D}$. As a result, the model is described by an exponential number of model parameters $\vect{w}$. 
This motivates the introduction of the second assumption. For the CPD, the following theorem holds.
\begin{theorem}[CPD kernel machine~\citep{QTNM_Wesel_2024}]
\label{CPD_km}
    \textit{Suppose} $\tenOp{\vect{w}, I_1, I_2, \dots, I_D}$ \textit{is a tensor in CPD form. The model responses and associated gradients of}
    \begin{equation*}
        f(\vect{x}) = \sqBr{\vect{\phi}^{(D)}(x_D) \otimes \dots \otimes \vect{\phi}^{(1)}(x_1)}^\top \vect{w}
    \end{equation*}
    \textit{can be computed in} $\mathcal{O}(DIR)$ \textit{rather than} $\mathcal{O}(\prod_{d=1}^D I_d)$\textit{, where} $I = \max(I_1, \dots, I_D)$.
\end{theorem}

Training a CPD kernel machine means solving the following nonlinear nonconvex optimization problem:
\begin{equation}\label{cpd_km_opt}
    \begin{split}
        &\min_{\vect{v}}\frac{1}{2} \|\vect{y} - \mat{\Phi}\vect{w}(\vect{v})\|_2^2 + \frac{\alpha}{2}\|\vect{w}(\vect{v})\|_2^2 \\
        & \text{s.t. }\vect{w}(\vect{v}) = \sum_{r=1}^{R} \vect{v}^{(D)}_{r} \otimes \dots \otimes \vect{v}^{(1)}_{r},
    \end{split} 
\end{equation}
where $\vect{y}$ is the target variable vector, $\Phi_{ni} \coloneq \phi_i(\vect{x}_n)$, $\vect{x}_n$ represents an $n$th row of data matrix $\mat{X} \inCo{N\times D}$, $\vect{v}^{(d)}_{r}$ is the $r$th column of the $d$th CPD core $\mat{V}^{(d)}$, $\vect{v}$ denotes the vector containing all entries of the CPD cores, and $\alpha$ is a regularization hyperparameter.
Note that, to compute the model response $\vect{f} \coloneq \mat{\Phi} \vect{w}$, we implicitly take its real part, as the features $\vect{\phi}$ and weights $\vect{w}$ are generally complex.

Common approaches for optimizing tensor network  models include specialized algorithms such as alternating least squares (ALS)~\citep{ ALS_Uschmajew_2012, TD_FF_Wesel_2021} and Riemannian optimization~\citep{EM_Novikov_2018}, which exploit the multilinear tensor structure for efficient convergence. Generic first- or second-order gradient-based methods can also be employed.

\subsection{Bayesian Inference}
Bayesian inference models uncertainty over parameters using probability distributions via Bayes’ rule~\citep{PML_Intro_Murphy_2022}:
\begin{equation}\label{bayes}
    p(\vect{v} | \mathcal{D}) = \dfrac{p(\mathcal{D}|\vect{v})p(\vect{v})}{p(\mathcal{D})},
\end{equation}
where $p(\vect{v})$ is the prior, $p(\vect{v} | \mathcal{D})$ is the posterior, reflecting knowledge before and after observing the data, respectively; $p(\mathcal{D}|\vect{v})$ is the likelihood of the data given the parameters; $p(\mathcal{D})$ is the marginal likelihood, obtained by integrating out $\vect{v}$.
The posterior predictive distribution~\eqref{pd_general} can then be  obtained by marginalizing over parameters:
\begin{equation}\label{bma}
    p(y^*|\vect{x}^*, \mathcal{D}) = \int p(y^*|\vect{x}^*, \vect{v}) p(\vect{v} | \mathcal{D}) d\vect{v},
\end{equation}
where $p(y^*|\vect{x}^*, \vect{v})$ is the conditional model of outputs.
In the next section, we describe the assumptions and techniques enabling tractable Bayesian inference for tensor network kernel machines.

\section{BAYESIAN TENSOR NETWORK KERNEL MACHINES}
We consider the following discriminative model:
\begin{equation}\label{y_model}
    y = f(\vect{x}, \vect{v}) + e = \vect{\phi}(\vect{x})^\top\vect{g}(\vect{v}) + e,
\end{equation}
where $\vect{g}(\cdot): \mathbb{C}^{DIR} \rightarrow \mathbb{C}^{I^D}$ is the CPD parameterization mapping with $I \coloneq I_d$ for all $d=1\dots D$; $\vect{v} \coloneq \vecOp{\left[\vect{v}^{(1)}, \dots, \vect{v}^{(D)}\right]} \inCo{DIR}$, with $\vect{v}^{(d)} \coloneq \vecOp{\mat{V}^{(d)}} \inCo{IR}$; and 
$e \sim \mathcal{N}(0,\,\beta^{-1})$ is a Gaussian noise term with precision $\beta$.
Based on the model formulation in Equation~\eqref{y_model}, the conditional distribution of the outputs is given by:
\begin{equation}\label{cond_model}
    p(y|\vect{x}, \vect{v}, \beta) = \mathcal{N}(\vect{\phi}(\vect{x})^\top\vect{g}(\vect{v}),\,\beta^{-1}).
\end{equation}
Another key component of Bayesian inference is the prior distribution, which encodes initial beliefs and influences model uncertainty. Here, we define it as:
\begin{equation}\label{prior_weights}
    p(\vect{v}|\gamma) \coloneq \mathcal{N}(\vect{0},\,\gamma^{-1}\mat{I}),
\end{equation}
where $\gamma$ denotes the precision hyperparameter.
For tractable Bayesian inference, we introduce a variational posterior under the mean-field assumption:
\begin{equation}\label{q_var}
q(\vect{v}, \beta, \gamma) = q(\vect{v})q(\beta)q(\gamma),
\end{equation}
which allows efficient approximation of the true posterior. 
Details on variational inference and the hyperprior are provided in Section~\ref{vi_hyperprior} of the Appendix.

\subsection{Tensor Network Parameters Posterior}
A central element of Bayesian inference is the posterior distribution over model parameters~\eqref{bayes}, which captures both the model’s capacity and its uncertainty. Computing the posterior predictive distribution~\eqref{bma} requires evaluating a high-dimensional integral, which is often intractable. 
To this end, we employ the Laplace approximation~\citep{ML_Bishop_2006} to obtain a tractable posterior estimate, approximating $p(\vect{v} | \mathcal{D})$ by $q(\vect{v})$ as:
\begin{equation}\label{laplace}
    \begin{split}
        q(\vect{v}) &= \mathcal{N}(\vect{v}|\vect{v}^*,\,\mat{H}_{\vect{v}^*}^{-1}),\\
        \vect{v}^* &\coloneq \arg\min_{\vect{v}}  \mathcal{J}(\vect{v}),\\
        \mat{H}_{\vect{v}^*} &\coloneq \sqBr{\frac{\partial^2{ \mathcal{J}}}{\partial{\vect{v}}^\top \partial{\vect{v}}}}_{\vect{v} = \vect{v}^*}, \\
    \end{split}
\end{equation}
using the loss function $\mathcal{J}$, defined as follows:
\begin{equation}\label{loss}
    \mathcal{J}(\vect{v}) \coloneq \dfrac{\beta}{2} \|\vect{y} -\mat{\Phi}\vect{g}(\vect{v})\|_2^2 + \dfrac{\gamma}{2}\|\vect{v}\|_2^2,
\end{equation}
where $\beta$ and $\gamma$ denote the precision hyperparameters.
The optimal model parameters $\vect{v}^*$ in the minimization problem~\eqref{laplace} can be found using the ALS algorithm~\citep{TD_FF_Wesel_2021}, where each CPD core $\mat{V}^{(d)}$ is updated sequentially for $d=1, \dots, D$ while keeping the other cores fixed:
\begin{equation}\label{wk_update_specs}
    \begin{split}
        \vecOp{\mat{V}^{(d)}} &= \rBr{\mat{A}^{(d)\top} \mat{A}^{(d)} + \dfrac{\gamma}{\beta} \mat{I}}^{-1}\mat{A}^{(d)\top}\vect{y}, \\
        \mat{A}^{(d)} &\coloneq \rBr{\mat{Z}^{(d)} \odot_R \mat{\Phi}^{(d)}} \inCo{N \times IR}, \\
        \mat{Z}^{(d)} &\coloneq \bigoasterisk_{\substack{k=1 \\ k\neq d}}^D \mat{\Phi}^{(k)} \mat{V}^{(k)} \inCo{N \times R}, \\
        \Phi^{(d)}_{ni} &\coloneq \vect{\phi}_{i}^{(d)}(x_{nd}).
    \end{split}  
\end{equation}
This procedure is repeated for several epochs (sweeps).

\subsection{Hessian Matrix Approximations}
A critical and technically challenging aspect of the Laplace approximation is the computation of the Hessian matrix, $\mat{H}_{\vect{v}^*}$, as defined in Equation~\eqref{laplace}. 
Building on recent advances in Hessian estimation for Bayesian Neural Networks (BNNs)~\citep{Last_Bayesian_Kristiadi_2020, Laplace_Redux_Daxberger_2021, FSP_Laplace_Bamler_2025}, we classify and compare several types of Hessian approximations within the tensor network modeling framework - Full, Generalized Gauss–Newton (GGN), Block-Diagonal (Block), Diagonal (Diag), and Last Core (Last) - in terms of their memory usage and training complexity.

\begin{table*}[!ht]
    \centering
    \caption{Comparison of memory and computational complexities across various Hessian approximation methods. For reference, the complexity of MAP estimate training is also provided. The notation used is as follows: $M$ denotes memory required to store parameters; $M_{\text{peak}}$ denotes peak memory usage during training; $T$ represents training complexity; $E$ is the number of ALS epochs; $N$ is the sample size; $D$ is the data dimensionality; $I$ is the number of basis functions per dimension; and $R$ is the CPD rank. We assume $N \gg IR$.}
    \label{table:hess_complexity}
    \begin{center}
    \scalebox{1.0}{\begin{tabular}{||l||c|c|c|c|c||}
    \toprule
     & MAP - $\vect{v}^*$ & Last & Diag & Block & GGN/Full\\
    \midrule
    \midrule
    $M$ & $\mathcal{O}(DIR)$ & $\mathcal{O}(I^2R^2)$ & $\mathcal{O}(DIR)$ & $\mathcal{O}(DI^2R^2)$ &  $\mathcal{O}((DIR)^2)$ \\
    $M_{\text{peak}}$ & $\mathcal{O}((N + D)IR)$ & $\mathcal{O}(NIR)$ & $\mathcal{O}((N+D)IR)$ & $\mathcal{O}(NIR + DI^2R^2)$ & $\mathcal{O}(NIR + (DIR)^2)$ \\
    $T$ & $\mathcal{O}(ENDI^2R^2)$ & $\mathcal{O}(NI^2R^2)$ & $\mathcal{O}(NDIR)$ & $\mathcal{O}(NDI^2R^2)$ & $\mathcal{O}(N(DIR)^2)$ \\
    \bottomrule
    \end{tabular}}
    \end{center}
\end{table*}

\paragraph{Full Hessian.}
Leveraging the multilinearity of the CPD kernel machine, we obtain the full Hessian, stated in the theorem below.
\begin{restatable}{theorem}{MainTheorem}
\label{full_hess}
\textit{The Hessian of the CPD kernel machine loss function} $\mathcal{J}(\vect{v})$~\eqref{loss} \textit{can be written in block-matrix form for} $k, m = 1, \dots, D$ \textit{and} $r, p = 1, \dots, R$ \textit{as follows:}
\begin{equation*}
\begin{cases}
    &\dfrac{\partial^2 \mathcal{J}}{\partial \rBr{\vect{v}^{(k)}}^\top \partial \vect{v}^{(k)}} = \beta\mat{A}^{(k)\top}\mat{A}^{(k)} + \gamma\mat{I},\\
    &\dfrac{\partial^2 \mathcal{J}}{\partial \rBr{\vect{v}^{(m)}_p}^\top \partial \vect{v}^{(k)}_r} = \beta \mat{\Phi}^{(k)\top} \mat{T}^{(k, m, r, p)} \mat{\Phi}^{(m)},
\end{cases}
\end{equation*}
\textit{with the intermediate matrices and using the Kronecker delta, $\delta_{ij}$:}
\begin{equation*}
    \begin{split}
        \mat{T}^{(k, m, r, p)} &\coloneq \mat{D}^{(k, m, r, p)} + \delta_{rp} \mat{E}^{(k, m, r)},\\
        \mat{D}^{(k, m, r, p)} &\coloneq 
        \diagOp{\sqBr{\mat{Z}^{(m)} \odot_R \mat{Z}^{(k)}}_{:(rp)}},\\
        \mat{E}^{(k, m, r)} &\coloneq
        \diagOp{\vect{f} - \vect{y}}
        \diagOp{\sqBr{\mat{Z}^{(k, m)}}_{:r}},\\
        \mat{Z}^{(k, m)} &\coloneq \bigoasterisk_{\substack{d=1 \\ d\neq k \\ d\neq m}}^D \mat{\Phi}^{(d)}\mat{V}^{(d)}\text{.}
    \end{split}
\end{equation*}
\end{restatable}
\begin{proof}
See Appendix~\ref{hessian_proofs}.
\end{proof}

\paragraph{GGN Hessian.}
The primary advantage of the generalized Gauss–Newton approximation lies in its positive semi-definiteness, in contrast to the full Hessian, which may contain both positive and negative eigenvalues~\citep{BNN_LLA_Immer_2021}. The formula for this approximation is given by:
\begin{equation}\label{ggn_hess}
    \mat{H}_{\vect{v}^*} = \beta\mat{A}^\top\mat{A} + \gamma\mat{I},
\end{equation}
where $\mat{A} \coloneq \sqBr{\mat{A}^{(1)}, \dots, \mat{A}^{(D)}} \inCo{N \times DIR}$ and  matrices $\mat{A}^{(d)}$ are defined in Equation~\eqref{wk_update_specs}. 

\paragraph{Block-Diagonal Hessian.}
The use of block-diagonal Hessian approximations has been effective in BNNs literature~\citep{KFAC_Martens_2015, GNN_DL_Botev_2017}. The key assumption is that the CPD cores (analogous to separate layers in BNNs) are independent, allowing the posterior to be expressed as:
\begin{equation}
    \begin{split}
        q(\vect{v}) &= \prod_{d=1}^D \mathcal{N}(\vect{v}^{(d)}|(\vect{v}^{(d)})^*,\,(\mat{H}^{(d)})^{-1}), \\
        \mat{H}^{(d)} &= \beta\mat{A}^{(d)\top}\mat{A}^{(d)} + \gamma\mat{I}.
    \end{split}
\end{equation}

\paragraph{Diagonal Hessian.}
In this case, we assume that all model weights are independent (mean-field approach~\citep{SP_BTN_Mandic_2022}), so the posterior is approximated by a diagonal multivariate Gaussian:
\begin{equation}
    \begin{split}
        q(\vect{v}) &= \prod_{d=1}^D \mathcal{N}(\vect{v}^{(d)}|(\vect{v}^{(d)})^*,\,(\mat{H}^{(d)})^{-1}), \\
        \mat{H}^{(d)} &= \diagOp{\beta\sqBr{\mat{A}^{(d)}\oasterisk\mat{A}^{(d)}}^\top\vect{1} + \gamma\vect{1}},
    \end{split}
\end{equation}
where $\vect{1}$ represents a vector of ones.

\paragraph{Last Core Hessian.}
Bayesian inference applied only to the last layer of a neural network was proposed by~\citet{Last_Bayesian_Kristiadi_2020} and shown to be both theoretically well-founded and empirically effective. In our setting, this corresponds to the last CPD core, $\vect{v}^{(D)} = \vecOp{\mat{V}^{(D)}}$, with the posterior over the weights given by
\begin{equation}\label{eq_last_core}
    \begin{split}
        q(\vect{v}) &= \hat{q}(\vect{v}) \mathcal{N}(\vect{v}^{(D)}|(\vect{v}^{(D)})^*,\,(\mat{H}^{(D)})^{-1}), \\
        \mat{H}^{(D)} &= \beta\mat{A}^{(D)\top}\mat{A}^{(D)} + \gamma\mat{I},
    \end{split}
\end{equation}
where $\hat{q}(\vect{v}) = \rBr{\prod_{d=1}^{D-1}\delta(\vect{v}^{(d)} - (\vect{v}^{(d)})^*)}$ and $\delta$ denotes the Dirac delta function.

Table~\ref{table:hess_complexity} compares different Hessian approximation methods in terms of three components: memory $M$ required to store model parameters (mean and covariance), peak memory $M_{\text{peak}}$ during training, and training complexity $T$ in number of operations. For completeness, we also include the complexity of the ALS procedure~\eqref{wk_update_specs}.
In summary, to achieve the same computational complexity as MAP estimation, one should use the Last, Diagonal, or Block approximations. 
For high-dimensional data ($D \gg 1$), the last-core Hessian requires the least memory. In Section~\ref{experiments}, we empirically compare these approximation strategies by their impact on predictive performance.

\paragraph{Eigenvalue Decomposition}
Another challenge is the inversion of the Hessian matrix. Prior work~\citep{SWAG_Maddox_2019} shows that the Hessian is often poorly conditioned, with many eigenvalues being zero or negative. Consequently, the optimization problem in Equation~\eqref{laplace} is ill-conditioned and difficult to solve. 
To address this, we apply a truncated eigenvalue decomposition to the Hessian, $\mat{H} = \mat{U} \diagOp{\vect{\lambda}} \mat{U}^\top$, with eigenvalues sorted in descending order:
\begin{equation}
\hat{\mat{H}} \coloneq \sum_{j=1}^{\hat{R}} \lambda_{j}\vect{u}_{j}\vect{u}_{j}^{\top},
\end{equation}
where $\vect{u}_{j}$ is the $j$th column of $\mat{U}$, $\lambda_j \geq \hat{t}$ for $j=1,\dots,\hat{R}$, and $\lambda_{\hat{R}+1} < \hat{t}$.
The thresholding hyperparameter $\hat{t}$ is selected via cross-validation, with candidate values ranging from the smallest to the largest eigenvalue of the Hessian $\mat{H}$. One may ask why an additional hyperparameter is needed when $\gamma$ already controls both the conditioning of the optimization problem and the Hessian eigenvalues. The issue is that excessively large $\gamma$ can drive the CPD tensor-product structure toward degenerate (near-zero) solutions. Introducing $\hat{t}$ decouples these roles: $\gamma$ governs MAP estimation, while $\hat{t}$ regulates Hessian inversion.

\begin{figure*}[ht!]
\includegraphics[width=170mm]{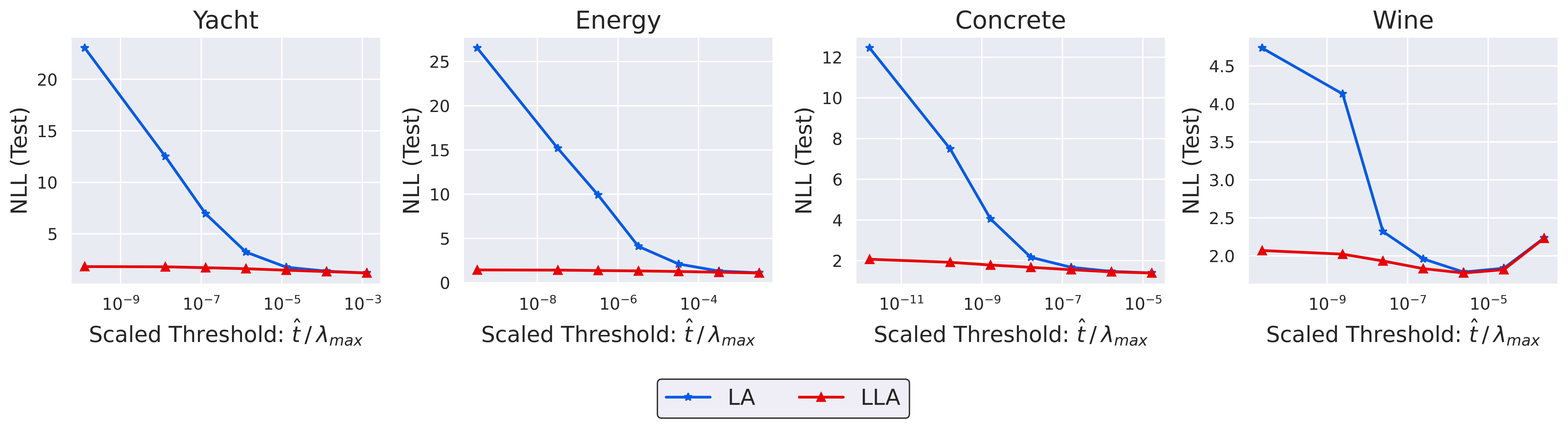}
\caption{
Test NLL performance of the LA-TNKM model with LA (Laplace Approximation) and LLA (Linearized Laplace Approximation) predictive distributions (blue and red curves respectively) as a function of the thresholding hyperparameter $\hat{t}$ (logarithmic x-axis) scaled by the largest eigenvalue of the full Hessian matrix for different real-life datasets (column-wise). The results consistently highlight the advantage of using the linearized Laplace approximation over the original formulation, in line with prior findings for BNNs~\citep{BNN_LLA_Immer_2021}.
}
\label{plot_linearization}
\end{figure*}

\subsection{Predictive Distribution}
Regardless of the chosen posterior approximation (e.g., the Hessian form), the predictive distribution is obtained by integrating the approximate posterior $q(\vect{v})$ with the conditional model:
\begin{equation}\label{pd_la}
    \begin{split}
        p_{LA}(y|\vect{x}, \mathcal{D}) &\coloneq \mathbb{E}_{q(\vect{v})}\sqBr{p(y|f(\vect{x}, \vect{v}))} \\
        & \approx \dfrac{1}{S}\sum_{s=1}^S p(y|f(\vect{x}, \vect{v}_s))\text{, } \vect{v}_s \sim q(\vect{v}),
    \end{split}
\end{equation}
where the (intractable) expectation is approximated via Monte Carlo sampling, since $f$ depends nonlinearly on $\vect{v}_s$. 
When using the GGN approximation to the Hessian, the model $f(\vect{x}, \vect{v})$ can be interpreted as a generalized linear model (GLM) with a local linearization~\citep{BNN_LLA_Immer_2021, LLA_Evidence_Daxberger_2022} of the form:
\begin{equation}\label{linear_f}
    f^{lin}_{\vect{v}^*}(\vect{x}, \vect{v}) \coloneq f(\vect{x}, \vect{v}) + \vect{g}_{\vect{v}^*}(\vect{x})^\top(\vect{v} - \vect{v}^*),
\end{equation}
where $\vect{g}_{\vect{v}^*}(\vect{x}) \inCo{DIR}$ is the gradient of the function $f$ evaluated at point $\vect{v}^*$. Since the GGN approximation for posterior covariance corresponds to the posterior of the linearized model~\eqref{linear_f}, we can use this model for prediction, consistent with GGN-based inference~\citep{BNN_LLA_Immer_2021}:
\begin{equation}\label{pd_linearized_la}
    p_{LLA}(y|\vect{x}, \mathcal{D}) \coloneq \mathbb{E}_{q(\vect{v})}\sqBr{p(y|f^{lin}_{\vect{v}^*}(\vect{x}, \vect{v}))}.
\end{equation}
We denote the standard predictive distribution as LA and its linearized counterpart as LLA.
In Section~\ref{lla_vs_la}, we empirically compare their predictive performance for tensor network kernel machines.

\begin{figure*}[ht!]
\includegraphics[width=170mm]{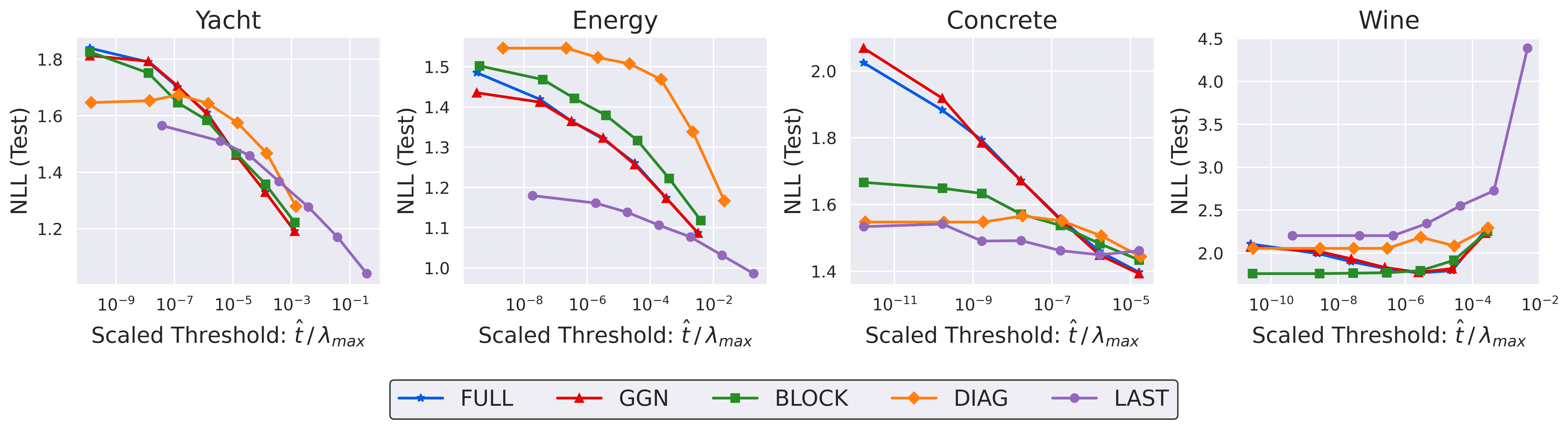}
\caption{
Test NLL performance of the LA-TNKM model using LLA (Linearized Laplace Approximation) predictive distribution with five approximations of the Hessian as a function of the thresholding hyperparameter $\hat{t}$ (logarithmic x-axis) scaled by the largest eigenvalue of the corresponding Hessian matrix for different real-life datasets (column-wise). The results consistently show that Hessian choice is data-dependent. However, Last and Block approximations (purple and green colors respectively) show better performance on average.
}
\label{plot_hess}
\end{figure*}

\section{NUMERICAL EXPERIMENTS}\label{experiments}
For all experiments, the input $\vect{x}$ and output $\vect{y}$ are standardized. A unit-norm polynomial feature mapping~\citep{SP_BTN_Mandic_2022} is applied as the local mapping $\vect{\phi}^{(d)}$ in Equation~\eqref{tensor_product_features} to improve training stability. No normalization was applied to the synthetic dataset in Section~\ref{subsec:ill_example}. We set the number of basis functions uniformly across all $d=1, \dots, D$ as $I_d = I$ to simplify hyperparameter tuning. For real-data experiments, we use nine UCI regression datasets~\citep{UCI}. For each dataset, 90\% of the data is randomly chosen for training, with the remaining 10\% reserved for testing. 
We evaluate probabilistic performance using the test negative log-likelihood (NLL; lower is better). While NLL is used as the primary metric in the main text, we also assess predictive accuracy and uncertainty calibration using RMSE, ECP-95, WCPI-95, and RCE. Formal definitions are provided in Appendix~\ref{app:metrics}, and full experimental results are reported in Appendix~\ref{app:exp_res}. 
The source code and data required to reproduce all experiments are publicly available on GitHub.\footnote{\url{https://github.com/AlbMLpy/laplace-tnkm}}

\begin{figure*}[ht!]
\includegraphics[width=170mm]{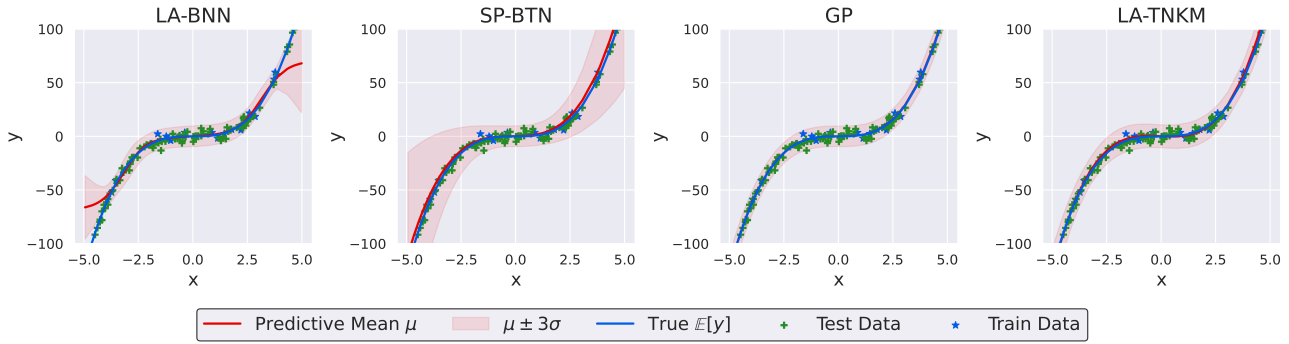}
\caption{
Predictive distributions on the synthetic dataset ($y = x^3 + \epsilon $, where $\epsilon \sim \mathcal{N}(0, 3^2)$) for LA-BNN, SP-BTN, GP, and the proposed LA-TNKM (from left to right). Compared to the baselines, LA-TNKM more closely matches GP regression behavior, particularly in terms of predictive uncertainty.
}
\label{fig:x3}
\end{figure*}

\subsection{Ablation Study}
In this set of experiments, we investigate key design choices in the development of the LA-TNKM model and their effect on predictive performance. First, we compare the original Laplace approximation~\eqref{pd_la} with its linearized variant~\eqref{pd_linearized_la}. Second, in addition to the theoretical complexities summarized in Table~\ref{table:hess_complexity}, we empirically evaluate different Hessian approximations. To this end, we measure the NLL on four real-world datasets while systematically varying the thresholding hyperparameter $\hat{t} = 10^{-5}, 10^{-3}, 10^{-2}, 10^{-1}, 1, 10, 100$.
All other hyperparameters are fixed. Figures~\ref{plot_linearization} and~\ref{plot_hess} show thresholding values scaled by the largest eigenvalue of each Hessian, accounting for spectral differences across Hessians.

\textbf{Q1: Is Linearization Useful for LA-TNKM?}\label{lla_vs_la}
The short answer is yes. Figure~\ref{plot_linearization} shows the comparison of linearized and nonlinearized predictive distributions on four data sets. As we can see, the linearized model consistently outperforms the original model on test NLL. This behavior can be explained by a mismatch between the predictive conditional model~\eqref{cond_model} and Laplace-GGN posterior~\eqref{ggn_hess} as pointed out in case of BNNs~\citep{BNN_LLA_Immer_2021}.

\textbf{Q2: What Hessian Approximation to Choose?}
There is no short answer here. As shown in Figure~\ref{plot_hess}, the last-core Hessian approximation achieves the best performance on the Yacht and Energy datasets, but performs poorly on the Wine dataset, where the Block approximation yields the best results. Furthermore, the thresholding hyperparameter $\hat{t}$ plays a key role, with higher values generally improving NLL across most datasets. This suggests that the approximated Hessian spectrum contains very small eigenvalues which, when inverted to form the covariance matrix, yield excessively large predictive variances. As a result, test performance degrades, as reflected in higher NLL. We recommend using the most efficient Hessian approximation (Last or Block, Table~\ref{table:hess_complexity}) with the hyperparameter $\hat{t}$ tuned via cross-validation. Furthermore, Table~\ref{table:uci-hessian} in Appendix~\ref{app:ablation} demonstrates that the last-core Hessian consistently yields better data-fitting performance (lower RMSE) and improved uncertainty estimation (lower NLL and RCE) across nine UCI datasets.

\subsection{Uncertainty on Synthetic Data}\label{subsec:ill_example}
To evaluate the uncertainty estimation capability of the proposed model, we conducted experiments on a synthetic dataset. We compared the LA-TNKM method with three baselines: Bayesian neural network using Laplace approximation (see Section~\ref{app:baselines}); structured posterior Bayesian tensor network (SP-BTN)~\citep{SP_BTN_Mandic_2022}, a tensor-based method using variational inference for posterior approximation; and GP regression~\citep{GP_Rasmussen_2006}. The experiment followed the one-dimensional regression setup from~\citep{SP_BTN_Mandic_2022}, with data generated according to $y = x^3 + \epsilon $, where $\epsilon \sim \mathcal{N}(0, 3^2)$. The training set comprised 20 samples drawn uniformly from the interval $[-4, 4]$, and the test set included 100 samples from $\mathcal{U}[-5, 5]$. For all models, we fixed the noise precision to the ground truth value $\beta = \frac{1}{9}$, set the CPD rank to $R = 2$, and used a local dimension of $I = 4$. For SP-BTN model we used $M_R = 2$ and for LA-TNKM we used the last-core approximation~\eqref{eq_last_core} with linearized LA~\eqref{pd_linearized_la}.
Figure~\ref{fig:x3} shows that LA-TNKM performs on par with GP regression, which serves as a strong benchmark. In contrast, SP-BTN tends to overestimate predictive variance, resulting in overly broad and less informative 
confidence intervals, likely due to underfitting in low-data regimes. By comparison, LA-BNN exhibits well-calibrated behavior in regions supported by training data, but shows signs of model mismatch when extrapolating beyond the observed data range. This behavior is further supported by the quantitative results in Table~\ref{table:x-cube-comparison} in Appendix~\ref{app:synthetic}, where the corresponding evaluation metrics are reported.

\begin{table*}[!ht]
\caption{Average test NLL on several UCI regression datasets. We train on random 90\% of the data and predict on 10\%. Each experiment is repeated 10 times, and we report the mean ± standard deviation. $N$ is the sample size, and $D$ is the data dimensionality. 'N/A' refers to cases where reporting a value is not feasible due to prohibitively large computational demands or lack of numerical stability. Among the evaluated methods, LA-TNKM (Last) ranks highest on four of the nine datasets and performs comparably on the rest.}
\label{table:uci-comparison}
\begin{center}
\scalebox{0.7}{
\begin{tabular}{||l|c|c||c|c|c|c|c|c|c|c||c||}
\toprule
Dataset & N & D & GWI-DNN & FVI & BBB & $\alpha$=0.5 & EXACT GP & LA-BNN & MF-BTN & SP-BTN & LA-TNKM \\
\midrule
BOSTON & 506 & 13 & 2.27±0.06 & 2.33±0.04 & 2.76±0.04 & 2.45±0.02 & 2.46±0.04 & 2.10±0.17 & 1.45±0.11 & 1.19±0.07 & \textbf{0.95±0.16} \\
CONCRETE & 1030 & 8 & 2.64±0.06 & 2.88±0.06 & 3.28±0.01 & 3.06±0.03 & 3.05±0.02 & 1.56±0.05 & 1.41±0.07 & \textbf{0.74±0.06} & 0.82±0.09 \\
ENERGY & 768 & 8 & 0.91±0.12 & 0.58±0.05 & 2.17±0.02 & 0.95±0.09 & 0.54±0.02 & 1.32±0.02 & 1.41±0.04 & 0.62±0.07 & \textbf{-1.40±0.04} \\
KIN8NM & 8192 & 8 & \textbf{-1.2±0.03} & -1.15±0.01 & -0.81±0.01 & -0.92±0.02 & N/A±0.00 & 1.22±0.01 & 1.42±0.01 & 0.89±0.02 & 0.48±0.02 \\
NAVAL & 11934 & 16 & -6.76±0.1 & \textbf{-7.21±0.06} & -2.80±0.00 & -2.97±0.14 & N/A±0.00 & N/A±0.00 & 1.42±0.01 & 1.37±0.02 & -1.16±0.64 \\
POWER & 9568 & 4 & 2.74±0.02 & 2.69±0.00 & 2.83±0.01 & 2.81±0.00 & N/A±0.00 & 0.88±0.02 & 1.42±0.02 & 0.57±0.01 & \textbf{-0.01±0.03} \\
PROTEIN & 45730 & 9 & 2.87±0.0 & 2.85±0.00 & 3.00±0.00 & 2.90±0.00 & N/A±0.00 & \textbf{1.00±0.04} & 1.42±0.01 & 1.36±0.01 & 1.06±0.01 \\
RED WINE & 1588 & 11 & 0.76±0.08 & 0.97±0.06 & 1.01±0.02 & 1.01±0.02 & \textbf{0.26±0.03} & 2.03±0.11 & 1.35±0.06 & 1.26±0.05 & 1.24±0.06 \\
YACHT & 308 & 6 & 0.29±0.1 & 0.59±0.11 & 1.11±0.04 & 0.79±0.11 & 0.10±0.05 & 1.34±0.02 & 1.38±0.10 & 0.58±0.34 & \textbf{-0.52±0.17} \\
\bottomrule
\end{tabular}
}
\end{center}
\end{table*}

\subsection{UCI Regression}\label{subsec:uci_reg}
To evaluate LA-TNKM against existing probabilistic models, we compare it with several baselines on a range of real-world datasets. 
We consider several weight-space BNN approaches, including Bayes-by-Backprop (BBB)~\citep{BBB_Blundell_2015}, variational alpha-dropout ($\alpha = 0.5$)~\citep{Dropout_BNN_Li_2017} and Laplace approximation (LA-BNN) (see Section~\ref{app:baselines}).
In addition, we include GP regression~\citep{GP_Rasmussen_2006} as a reference and compare two function-space BNN inference methods: functional variational inference (FVI)~\citep{FVI_Ma_2021} and Gaussian Wasserstein inference BNN (GWI-DNN)~\citep{GVI_DL_Wild_2022}. For a fair comparison, we also include two tensor-based regression models: MF-BTN and SP-BTN~\citep{SP_BTN_Mandic_2022}. The optimal hyperparameters of LA-TNKM, MF-BTN, and SP-BTN — such as the rank $R$, polynomial order $I$, and thresholding parameter $\hat{t}$ — are determined via cross-validation. For the Laplace approximation, we used the linearized model for prediction~\eqref{pd_linearized_la} and the last-core Hessian~\eqref{eq_last_core} to approximate the posterior.

In Table~\ref{table:uci-comparison}, we present the average test NLL for all compared methods. The results for the first five baselines are taken from~\citep{GVI_DL_Wild_2022}. The experimental results demonstrate that LA-TNKM outperforms other methods on four out of nine datasets, highlighting its effectiveness. However, performance differs across datasets, largely due to the choice of the feature map $\vect{\phi}(\vect{x})$, which plays a critical role in determining the model’s approximation and extrapolation behavior. If the underlying function cannot be adequately represented with a low-order polynomial basis (i.e., when $I$ is small), approximation errors accumulate, degrading performance. 
In addition, Tables~\ref{table:uci-rmse-rce} and~\ref{table:uci-ecp-wcpi} in Appendix~\ref{app:uci-reg} report the complete set of evaluation metrics (RMSE, ECP-95, WCPI-95, RCE). Across most datasets, LA-TNKM achieves better uncertainty calibration, with lower RCE and coverage closer to the nominal level (ECP-95), as well as smaller WCPI-95 values, indicating sharper predictive intervals. Overall, these results suggest that LA-TNKM balances calibration and precision, providing reliable yet informative uncertainty estimates compared to the other baselines.

\section{RELATED WORK}
Our contributions lie at the intersection of BNNs, their connections to GPs, and probabilistic methods based on tensor networks (TNs). 
\citet{LA_BNN_GP_Khan_2019} establish a theoretical link between BNNs and GPs, showing that Gaussian posterior approximations — obtained via Laplace or variational inference — correspond to GP regression posteriors. Local linearization and its connection to the GGN approximation in BNNs is explored by \citet{BNN_LLA_Immer_2021}, and further extended in \cite{LLA_Evidence_Daxberger_2022}, where the authors address mismatches between classical Laplace approximation assumptions and the behavior of modern neural networks. An alternative perspective is presented by \citet{Bayes_No_Underfit_Miani_2025}, who propose a scalable, matrix-free method for constructing Bayesian approximations in the Hessian’s null space to mitigate underfitting.

A related line of work focuses on probabilistic TN models. \citet{BayesianMethod_TNs_Draper_2021} apply the Laplace approximation to tensor train models with Bayesian priors over the network parameters. In contrast, \citet{SP_BTN_Mandic_2022} propose a scalable variational inference framework for CPD-based models, using low-rank and Kronecker-structured posteriors to balance expressiveness and tractability. Complementing these approaches, \citet{TN_GPR_Menzen_2023} approximate GPs by projecting the problem into a low-dimensional subspace defined by a TN, performing Bayesian inference there, and then projecting back to the original space for GP predictions.

\section{CONCLUSION}
In this paper, we developed a novel Bayesian Tensor Network Kernel Machine, LA-TNKM, which provides uncertainty estimates for its predictions.
We applied the Laplace approximation to the CPD weight posterior distribution to make Bayesian inference tractable and introduced several types of Hessian matrix approximations, highlighting their respective advantages and limitations. 
We experimentally validated the benefits of local linearization of the prediction function and compared the proposed LA-TNKM model against GP- and BNN-based baselines. The results demonstrate the competitiveness and effectiveness of LA-TNKM across a diverse range of datasets and applications. 
Future works may focus on improving the optimization strategy (finding MAP estimate), designing problem-dependent priors, and exploring alternative tensor network architectures (e.g., Hierarchical Tucker) to enhance flexibility and performance.

\begin{acknowledgements} 
    This publication is part of the project Sustainable learning for Artificial Intelligence from noisy large-scale data (with project number VI.Vidi.213.017) which is financed by the Dutch Research Council (NWO).
\end{acknowledgements}

\bibliography{bibfile}

\newpage

\onecolumn

\title{Laplace Approximation for Bayesian Tensor Network Kernel Machines\\(Supplementary Material)}
\maketitle

\appendix

\section{BAYESIAN TENSOR
NETWORK KERNEL MACHINES}

\subsection{Hessian Matrix Approximations}\label{hessian_proofs}
Before presenting the proof of the full Hessian theorem, we first establish three auxiliary lemmas.
\begin{lemma}\label{lemma_1}
   \textit{Exploiting the multilinearity property, the response of the CPD kernel machine can be expressed as following for} $d = 1, \dots, D$\textit{:} 
   \begin{equation*}
   \begin{split}
       \vect{f} &= \mat{\Phi}\vect{g}\rBr{\vect{v}} = \mat{A}^{(d)}\vect{v}^{(d)} = \mat{Z}\vect{1},\\
       \mat{Z} &\coloneq \bigoasterisk_{\substack{d=1}}^D \mat{\Phi}^{(d)} \mat{V}^{(d)}\text{.}
    \end{split}
   \end{equation*}
\end{lemma}
\begin{proof}
We begin the proof by expressing the response of the CPD kernel machine in element-wise form:
\begin{equation*}
    \begin{split}
        f_{n} &= \sqBr{\mat{\Phi}\vect{g}\rBr{\vect{v}}}_{n}\\
        &= \sum_{i=1}^{I^D}\Phi_{ni}g_i\rBr{\vect{v}}\\
        &=\sum_{i_1}^{I}\dots\sum_{i_D}^{I}\Phi_{ni_1}^{(1)} \dots \Phi_{ni_D}^{(D)}\sum_{r=1}^R V_{i_1r}^{(1)} \dots V_{i_Dr}^{(D)}\\
        &=\sum_{i_d}^{I}\sum_{r=1}^R \Phi_{ni_d}^{(d)} V_{i_dr}^{(d)} \sqBr{\mat{\Phi}^{(1)}\mat{V}^{(1)}}_{nr} \dots \sqBr{\mat{\Phi}^{(d-1)}\mat{V}^{(d-1)}}_{nr} \sqBr{\mat{\Phi}^{(d+1)}\mat{V}^{(d+1)}}_{nr} \dots \sqBr{\mat{\Phi}^{(D)}\mat{V}^{(D)}}_{nr}\\
        &=\sum_{i_d}^{I}\sum_{r=1}^R \Phi_{ni_d}^{(d)} V_{i_dr}^{(d)} \sqBr{\bigoasterisk_{\substack{k=1 \\ k\neq d}}^D \mat{\Phi}^{(k)} \mat{V}^{(k)}}_{nr}\\
        &=\sum_{i_d}^{I}\sum_{r=1}^R \Phi_{ni_d}^{(d)} V_{i_dr}^{(d)} Z_{nr}^{(d)}\text{.}
    \end{split}
\end{equation*}
On the one hand, the last expression can be reformulated as:
\begin{equation*}
        f_n = \sum_{i_d}^{I}\sum_{r=1}^R \Phi_{ni_d}^{(d)} V_{i_dr}^{(d)} Z_{nr}^{(d)}
        = \sum_{r=1}^R \sqBr{\mat{\Phi}^{(d)}\mat{V}^{(d)}}_{nr} Z_{nr}^{(d)}
        = \sum_{r=1}^R Z_{nr}
        = \sqBr{\mat{Z}\vect{1}}_n\text{.}
\end{equation*}
On the other hand, applying the definitions of the row-wise Khatri–Rao product~\citep{TN1_Cichocki_2016} and vectorization Definition~\ref{def_vec} yields the following reformulation:
\begin{equation*}
    f_n = \sum_{i_d}^{I}\sum_{r=1}^R \Phi_{ni_d}^{(d)} V_{i_dr}^{(d)} Z_{nr}^{(d)}
    = \sum_{(i_dr) = 1}^{IR}\sqBr{\mat{Z}^{(d)} \odot_R \mat{\Phi}^{(d)}}_{n(i_dr)}v_{(i_dr)}^{(d)} = \sqBr{\sqBr{\mat{Z}^{(d)} \odot_R \mat{\Phi}^{(d)}}\vect{v}^{(d)}}_{n} = \sqBr{\mat{A}^{(d)}\vect{v}^{(d)}}_{n}\text{.}
\end{equation*}
\end{proof}

\begin{lemma}\label{lemma_2}
\textit{Suppose} $\mat{F} \coloneq \sqBr{\mat{B} \oasterisk \mat{C}\mat{V}} \odot_R \mat{A} $, $\vecOp{\mat{V}} \coloneq \vect{v}$. \textit{Then the following equation holds:}
   \begin{equation*}
\dfrac{\partial\sqBr{\mat{F}^\top\vect{z}}_{(ir)}}{\partial v_{(jp)}} = \delta_{rp}\sqBr{\mat{A}^\top\mat{H}^{(r)}\mat{C}}_{ij}\text{ , }
   \end{equation*}
\textit{with the intermediate matrix:}
\begin{equation*}
    \begin{split}
    &\mat{H}^{(r)} \coloneq \diagOp{\vect{z}}\diagOp{\sqBr{\mat{B}}_{:r}}.
    \end{split}
\end{equation*}
\end{lemma}
\begin{proof}
We obtain the following expression for the target function, $\sqBr{\mat{F}^\top\vect{z}}_{(ir)}$, using the definitions of the row-wise Khatri–Rao product and the Hadamard product~\citep{TN1_Cichocki_2016}:
\begin{equation}\label{Ftz_def}
    \sqBr{\mat{F}^\top\vect{z}}_{(ir)} = \sum_{n=1}^N z_n F_{n(ir)} = \sum_{n=1}^N z_n A_{ni} B_{nr} \sqBr{\mat{C}\mat{V}}_{nr}\text{.}
\end{equation}
Using the Kronecker delta $\delta_{ij}$ and the Leibniz product rule, the derivative of the target function in Equation~\eqref{Ftz_def} is given by:
\begin{equation}
\begin{split}
&\dfrac{\partial\sqBr{\mat{F}^\top\vect{z}}_{(ir)}}{\partial v_{(jp)}}\\
=&\text{ } \delta_{rp} \sum_{n=1}^N z_n A_{ni} B_{nr} C_{nj}\\
=&\text{ } \delta_{rp} \sqBr{\mat{A}^\top \diagOp{\vect{z}} \diagOp{\sqBr{\mat{B}}_{:r}}\mat{C}}_{ij}\\
=&\text{ } \delta_{rp} \sqBr{\mat{A}^\top \mat{H}^{(r)}\mat{C}}_{ij}\text{.}
\end{split}
\end{equation}
\end{proof}

\begin{lemma}\label{lemma_3}
\textit{Suppose} $\mat{F} \coloneq \sqBr{\mat{B} \oasterisk \mat{C}\mat{V}} \odot_R \mat{A}$ , $\vecOp{\mat{V}} \coloneq \vect{v}$, $\vecOp{\mat{Z}} \coloneq \vect{z}$. \textit{Then the following equation holds:}
\begin{equation*}
\dfrac{\partial\sqBr{\mat{F}^\top\mat{F}\vect{z}}_{(ir)}}{\partial v_{(jp)}} = \sqBr{\mat{A}^\top\mat{D}^{(r, p)}\mat{C}}_{ij} + \delta_{rp}\sqBr{\mat{A}^\top\mat{G}^{(r)}\mat{C}}_{ij}\text{ , }
   \end{equation*}
\textit{with the intermediate matrices:}
\begin{equation*}
    \begin{split}
    \mat{D}^{(r, p)} &\coloneq \diagOp{\sqBr{\rBr{\mat{A\mat{Z} \oasterisk \mat{B}}} \odot_R \rBr{\mat{C\mat{V} \oasterisk \mat{B}}}}_{:(rp)}}, \\
    \mat{G}^{(r)} &\coloneq \diagOp{\rBr{\mat{C}\mat{V} \oasterisk \mat{B} \oasterisk \mat{A}\mat{Z}}\vect{1}}\diagOp{\sqBr{\mat{B}}_{: r}}.
    \end{split}
\end{equation*}
\end{lemma}
\begin{proof}
We first express the elements of matrix $\mat{F}$ using the definitions of the row-wise Khatri–Rao product and the Hadamard product~\citep{TN1_Cichocki_2016}:
\begin{equation}\label{F_def}
    F_{n(ir)} = A_{ni} B_{nr} \sqBr{\mat{C}\mat{V}}_{nr}\text{.}
\end{equation}
Using Equation~\eqref{F_def}, we obtain the following expression for the target function:
\begin{equation}\label{FtFz_def}
    \sqBr{\mat{F}^\top\mat{F}\vect{z}}_{(ir)} = \sum_{(km) = 1}^{IR}z_{(km)}\sum_{n=1}^N A_{ni}A_{nk}B_{nr}B_{nm}\sqBr{\mat{C}\mat{V}}_{nr}\sqBr{\mat{C}\mat{V}}_{nm}\text{.}
\end{equation}
Using the Kronecker delta $\delta_{ij}$ and the Leibniz product rule, the derivative of the target function in Equation~\eqref{FtFz_def} is given by:
\begin{equation}\label{deriv_base}
\dfrac{\partial\left[\mat{F}^\top\mat{F}\vect{z}\right]_{(ir)}}{\partial v_{(jp)}} = \sum_{(km) = 1}^{IR}z_{(km)}\sum_{n=1}^N A_{ni}A_{nk}B_{nr}B_{nm}C_{nj}\rBr{{\sqBr{\mat{C}\mat{V}}_{nm}\delta_{rp}} + \sqBr{\mat{C}\mat{V}}_{nr}\delta_{mp}}\text{.}
\end{equation}
We begin with the first term of Equation~\eqref{deriv_base} and rewrite it using the relation $\vecOp{\mat{Z}} = \vect{z}$ as follows:
\begin{equation}\label{deriv_first}
\begin{split}
    &\sum_{(km) = 1}^{IR} z_{(km)}\sum_{n=1}^N A_{ni}A_{nk}B_{nr}B_{nm}C_{nj}\sqBr{\mat{C}\mat{V}}_{nm}\delta_{rp}\\
    =&\text{ }\delta_{rp}\sum_{n=1}^N\sum_{m = 1}^{R} A_{in}^\top C_{nj}\sqBr{\mat{C}\mat{V} \oasterisk \mat{B}}_{nm}\sqBr{\mat{A}\mat{Z}}_{nm}B_{nr}\\
    =&\text{ } \delta_{rp}\sum_{n=1}^N A_{in}^\top C_{nj}\sqBr{\rBr{\mat{C}\mat{V} \oasterisk \mat{B} \oasterisk \mat{A}\mat{Z}}\vect{1}}_{n}B_{nr}\\
    =&\text{ } \delta_{rp}\sqBr{\mat{A}^\top \diagOp{\rBr{\mat{C}\mat{V} \oasterisk \mat{B} \oasterisk \mat{A}\mat{Z}}\vect{1}}\diagOp{\sqBr{\mat{B}}_{: r}}\mat{C}}_{ij}\\
    =&\text{ } \delta_{rp}\left[\mat{A}^\top\mat{G}^{(r)}\mat{C}\right]_{ij}\text{.}
\end{split}
\end{equation}
The second term of Equation~\eqref{deriv_base} can be expressed as follows:
\begin{equation}\label{deriv_second}
\begin{split}
    &\sum_{(km) = 1}^{IR} z_{(km)}\sum_{n=1}^N A_{ni}A_{nk}B_{nr}B_{nm}C_{nj}\sqBr{\mat{C}\mat{V}}_{nr}\delta_{mp}\\
    =&\text{ } \sum_{n=1}^N \sum_{k = 1}^{I} Z_{kp}A_{ni}A_{nk}B_{nr}B_{np}C_{nj}\sqBr{\mat{C}\mat{V}}_{nr}\\
    =&\text{ } \sum_{n=1}^N A_{in}^\top \sqBr{\mat{A}\mat{Z}}_{np}B_{np}\sqBr{\mat{C}\mat{V} \oasterisk \mat{B}}_{nr} C_{nj}\\
    =&\text{ } \sqBr{\mat{A}^\top \diagOp{\sqBr{\rBr{\mat{A}\mat{Z} \oasterisk \mat{B}}\odot_R \rBr{\mat{C}\mat{V} \oasterisk \mat{B}}}_{:(rp)}} \mat{C}}_{ij}\\
    =&\text{ } \left[\mat{A}^\top\mat{D}^{(r, p)}\mat{C}\right]_{ij}\text{.}
\end{split}
\end{equation}

\end{proof}

\MainTheorem*  
\begin{proof}
We begin by proving the case of the diagonal blocks. By Lemma~\ref{lemma_1}, the loss function $\mathcal{J}$ can be rewritten for $k=1, \dots, D$ as:
\begin{equation}\label{lemma_1_loss}
    \mathcal{J}(\vect{v}) = \dfrac{\beta}{2} \|\vect{y} - \mat{A}^{(k)}\vect{v}^{(k)}\|_2^2 + \dfrac{\gamma}{2}\|\vect{v}\|_2^2\text{.}
\end{equation}
The first derivative of the $\ell_{2}$-regularized linear regression loss function is given by:
\begin{equation}\label{first_derivative}
    \dfrac{\partial \mathcal{J}}{\partial \vect{v}^{(k)}} = \beta\rBr{\mat{A}^{(k)\top} \mat{A}^{(k)}\vect{v}^{(k)} - \mat{A}^{(k)\top}\vect{y}} + \gamma \vect{v}^{(k)}\text{.}
\end{equation}
The second derivative w.r.t the same vector $\vect{v}^{(k)}$ can be computed as follows:
\begin{equation}\label{second_diag_derivative}
    \dfrac{\partial^2 \mathcal{J}}{\partial (\vect{v}^{(k)})^\top \partial \vect{v}^{(k)}} = \beta\mat{A}^{(k)\top}\mat{A}^{(k)} + \gamma\mat{I}\text{.}
\end{equation}
Using Equation~\eqref{first_derivative}, the general form of the second derivative with respect to $\vect{v}^{(m)}$ can be expressed as:
\begin{equation}\label{second_general_derivative}
    \dfrac{\partial^2 \mathcal{J}}{\partial (\vect{v}^{(m)})^\top \partial \vect{v}^{(k)}} = \beta \dfrac{\partial \sqBr{\mat{A}^{(k)\top} \mat{A}^{(k)}\vect{v}^{(k)}} }{\partial (\vect{v}^{(m)})^\top} - \beta \dfrac{\partial \sqBr{\mat{A}^{(k)\top}\vect{y}}}{\partial (\vect{v}^{(m)})^\top}\text{.}
\end{equation}
With $\mat{Z}^{(k,m)} \coloneq \bigoasterisk_{\substack{d=1 \\ d\neq k \\ d\neq m}}^D \mat{\Phi}^{(d)}\mat{V}^{(d)}$, the matrix $\mat{A}^{(k)}$ takes the form:
\begin{equation}\label{new_a}
    \mat{A}^{(k)} = \sqBr{\mat{Z}^{(k, m)} \oasterisk \mat{\Phi}^{(m)}\mat{V}^{(m)}} \odot_R \mat{\Phi}^{(k)}\text{.}
\end{equation}
Using Lemmas~\ref{lemma_1} and~\ref{lemma_3} together with Equation~\eqref{new_a} in element-wise form, the first term in Equation~\eqref{second_general_derivative} becomes:
\begin{equation}
    \dfrac{\partial \sqBr{\mat{A}^{(k)\top} \mat{A}^{(k)}\vect{v}^{(k)}}_{(ir)} }{\partial v^{(m)}_{(jp)}} = \sqBr{\mat{\Phi}^{(k)\top}\mat{D}^{(k, m, r, p)}\mat{\Phi}^{(m)}}_{ij} + \delta_{rp}\sqBr{\mat{\Phi}^{(k)\top}\mat{G}^{(k, m, r)}\mat{\Phi}^{(m)}}_{ij}
\end{equation}
with the intermediate matrices:
\begin{equation*}
    \begin{split}
        &\begin{split}
            \mat{D}^{(k, m, r, p)} &= \diagOp{\sqBr{\rBr{\mat{\Phi}^{(k)}\mat{V}^{(k)} \oasterisk \mat{Z}^{(k, m)}} \odot_R \rBr{\mat{\Phi}^{(m)}\mat{V}^{(m)} \oasterisk \mat{Z}^{(k, m)}}}_{:(rp)}} \\
            &= \diagOp{\sqBr{\mat{Z}^{(m)} \odot_R \mat{Z}^{(k)}}_{:(rp)}},
        \end{split}\\
        &\begin{split}
            \mat{G}^{(k, m, r)} &\coloneq \diagOp{\rBr{\mat{\Phi}^{(m)}\mat{V}^{(m)} \oasterisk \mat{Z}^{(k, m)} \oasterisk \mat{\Phi}^{(k)}\mat{V}^{(k)}}\vect{1}}\diagOp{\sqBr{\mat{Z}^{(k, m)}}_{: r}}\\
            &=\diagOp{\vect{f}}\diagOp{\sqBr{\mat{Z}^{(k, m)}}_{: r}}.
        \end{split}
    \end{split}
\end{equation*}
By applying Lemmas~\ref{lemma_1} and~\ref{lemma_2}, together with Equation~\eqref{new_a} and element-wise notation, the second term of Equation~\eqref{second_general_derivative} can be expressed as follows:
\begin{equation}
    \dfrac{\partial \sqBr{\mat{A}^{(k)\top} \vect{y}}_{(ir)} }{\partial v^{(m)}_{(jp)}} = \delta_{rp}\sqBr{\mat{\Phi}^{(k)\top}\mat{H}^{(k, m, r)}\mat{\Phi}^{(m)}}_{ij},
\end{equation}
where $\mat{H}^{(k, m, r)} \coloneq \diagOp{\vect{y}}\diagOp{\sqBr{\mat{Z}^{(k, m)}}_{:r}}$.

To conclude, we define new matrices as combinations of the previously introduced ones, leading to the final equations:
\begin{equation*}
    \begin{split}
        \mat{T}^{(k, m, r, p)} &= \mat{D}^{(k, m, r, p)} + \delta_{rp} \mat{E}^{(k, m, r)},\\
        \mat{E}^{(k, m, r)} &= \mat{G}^{(k, m, r)} - \mat{H}^{(k, m, r)}
        = \diagOp{\vect{f} - \vect{y}} \diagOp{\sqBr{\mat{Z}^{(k, m)}}_{:r}}.
    \end{split}
\end{equation*}
\end{proof}

\subsection{Variational Inference and Hyperprior}\label{vi_hyperprior}
Within a fully Bayesian framework, it is standard to place prior distributions over the unknown hyperparameters — specifically, the precision terms in our case. Following the approach outlined in~\citep{ML_Bishop_2006}, we model these hyperparameters using Gamma distributions:
\begin{equation}\label{prior_gamma_beta}
p(z) = \operatorname{Gam}(a^{(0)}_z, b^{(0)}_z), z\in\{\beta, \gamma \}.
\end{equation}

With these components in place, we are now able to define the variational posterior distribution under the mean-field assumption:
\begin{equation}\label{q_var_2}
q(\vect{v}, \beta, \gamma) = q(\vect{v})q(\beta)q(\gamma).
\end{equation}

To determine the functional forms of the variational factors, we sequentially employ the standard coordinate ascent update rule~\citep{ML_Bishop_2006}:
\begin{equation}\label{alt_bishop}
    \ln{q(\vect{z}_j)} = \mathop{\mathbb{E}}_{i \neq j}\left[\ln p(y, \vect{z})\right] + C\text{, } j=1, 2, 3,
\end{equation}
where $\vect{z} \coloneq [\vect{z_1}, z_2, z_3]$, with $\vect{z}_1 \coloneq \vect{v}$, $z_2 \coloneq \beta$, $z_3 \coloneq \gamma$; $p(y, \vect{z}) = p(y|\vect{v}, \beta)p(\vect{v}|\gamma)p(\beta)p(\gamma)$ is the joint distribution; $q(\vect{z}_j)$ denotes the optimal variational distribution for the $j$-th variable, and the expectation is taken over all other variables. 

The application of Equation~\eqref{alt_bishop} to the noise precision - $\beta$ leads to the following parameter update for its variational Gamma posterior for $t = 1, \dots, T$:
\begin{equation}\label{noise_precision}
    \begin{split}
    q(\beta) &= Gam(\beta|a_{\beta}^{(t)}, b_{\beta}^{(t)}), \\
    a_{\beta}^{(t)} &\coloneq a_{\beta}^{(t-1)} + \frac{N}{2}\text{, }\\
    b_{\beta}^{(t)} &\coloneq b_{\beta}^{(t-1)} + \frac{1}{2}\mathbb{E}_{\vect{v}}{\|\vect{y} -\mat{\Phi}\vect{g}(\vect{v})\|_2^2}.
    \end{split}
\end{equation}
A similar line of reasoning yields the variational update of the Gamma posterior parameters for the weights precision $\gamma$:
\begin{equation}\label{weights_precision}
    \begin{split}
    q(\gamma) &= Gam(\gamma|a_{\gamma}^{(t)}, b_{\gamma}^{(t)}), \\
    a_{\gamma}^{(t)} &\coloneq a_{\gamma}^{(t-1)} + \frac{DIR}{2}, \\
    b_{\gamma}^{(t)} &\coloneq b_{\gamma}^{(t-1)} + \frac{1}{2}\mathbb{E}{\vect{v}^\top\vect{v}}.
    \end{split}
\end{equation}
The posterior distribution over the model weights $\vect{v}$ can be obtained by applying Equations~\eqref{laplace} and~\eqref{loss}, where the precision hyperparameters $\beta$ and $\gamma$ are replaced with their expected values, $\mathbb{E}{\beta}$ and $\mathbb{E}{\gamma}$, respectively.


\section{ADDITIONAL EXPERIMENTAL RESULTS}\label{app:exp_res}

\subsection{Details of Computational Resources}
All experiments were conducted on a Dell Inc. Latitude 7440 laptop equipped with a 13th Gen Intel Core i7-1365U CPU and 16 GB of RAM.

\subsection{Details of Baseline Methods}\label{app:baselines}
\paragraph{LA-BNN.}
The model is a standard multilayer perceptron (MLP) with a single hidden layer and tanh activation~\citep{PBP_2015}, trained using the Adam optimizer~\citep{ADAM_Ba_2017}. The number of hidden units depends on the task: 32 for the synthetic experiments and 100 for the UCI datasets. To obtain probabilistic predictions, we apply the linearized Laplace approximation~\citep{BNN_LLA_Immer_2021}, treating all network parameters as random variables and thus employing the full-Hessian approximation.

\subsection{Evaluation Metrics}\label{app:metrics}

We evaluate both predictive accuracy and uncertainty quality using the following regression metrics.
Given inputs $\{\vect{x}_i\}_{i=1}^N$, targets $\{y_i\}_{i=1}^N$, predictive means $\{\hat{y}_i\}_{i=1}^N$, and predictive distributions $p(y|\vect{x})$, we define:

\paragraph{Root Mean Squared Error (RMSE).}
The RMSE measures point prediction accuracy and is defined as
\begin{equation}
\mathrm{RMSE} = \sqrt{\frac{1}{N} \sum_{i=1}^{N} (y_i - \hat{y}_i)^2}.
\end{equation}
Lower values indicate better predictive accuracy.

\paragraph{Negative Log-Likelihood (NLL).}
To assess probabilistic predictions, we compute the negative log-likelihood metric~\citep{UQM_2024},
\begin{equation}
\mathrm{NLL} = -\frac{1}{N} \sum_{i=1}^{N} \log p(y_i | \vect{x}_i),
\end{equation}
where $p(y_i| \vect{x}_i)$ denotes the predictive distribution.
This metric jointly evaluates accuracy and uncertainty calibration, with lower values indicating better performance.
It is also commonly referred to as the negative log predictive density (NLPD).

For a Gaussian predictive distribution, the negative log-likelihood is given by
\begin{equation}
\mathrm{NLL} = \frac{1}{N} \sum_{i=1}^{N}
\frac{1}{2} \log \big( 2\pi \sigma_i^2 \big)
+ \frac{(y_i - \hat{y}_i)^2}{2\sigma_i^2},
\end{equation}
where $\sigma_i^2$ denotes the predictive variance for a single test point.

\paragraph{Empirical Coverage Probability (ECP-{95}).}
The ECP-95 measures the fraction of targets lying within the predicted 95\% credible intervals~\citep{UQ_DL_2018},
\begin{equation}
\text{ECP-95} = \frac{1}{N} \sum_{i=1}^{N} \mathbb{I}
\left[ y_i \in \left[ l_i^{95}, u_i^{95} \right] \right].
\end{equation}
where $\mathbb{I}[\cdot]$ denotes the indicator function, which evaluates to $1$ if its argument is true and $0$ otherwise.
The quantities $l_i^{95}$ and $u_i^{95}$ denote the lower and upper bounds of the 95\% predictive credible interval for the $i$-th test point.
Values of ECP-95 closer to $0.95$ indicate better calibrated uncertainty.

In practice, the interval $\left[ l_i^{95}, u_i^{95} \right]$ is obtained empirically as the 2.5th and 97.5th percentiles of the predictive samples.

\paragraph{Width of Credible Prediction Interval (WCPI-95).}
The WCPI-95 quantifies the average width of the 95\% predictive intervals~\citep{Density_Reg_UQ_2023},
\begin{equation}
\text{WCPI-95} = \frac{1}{N} \sum_{i=1}^{N} (u_i^{95} - l_i^{95}),
\end{equation}
where $l_i^{95}$ and $u_i^{95}$ denote the lower and upper bounds of the 95\% predictive credible interval for the $i$-th test point.
Smaller values indicate sharper (less conservative) predictive uncertainty.

In practice, the values $l_i^{95}$ and $u_i^{95}$ are obtained empirically as the 2.5th and 97.5th percentiles of the predictive samples.

\paragraph{Relative Calibration Error (RCE).}
The relative calibration error measures the deviation between nominal and empirical coverage across multiple confidence levels~\citep{UQ_DL_2018}.
It is defined as
\begin{equation}
\mathrm{RCE} = \frac{1}{K} \sum_{k=1}^{K} \left| \hat{c}(\alpha_k) - \alpha_k \right|,
\end{equation}
where $K$ is the number of nominal coverage levels considered, $\alpha_k$ denotes the $k$-th nominal coverage level (e.g., 0.5, 0.6, ..., 0.95), and $\hat{c}(\alpha_k)$ is the corresponding empirical coverage computed over the test set.

Empirically, $\hat{c}(\alpha_k)$ can be obtained in the same way as ECP: for each test point $i$, compute the $\alpha_k$ predictive interval $[l_i^{\alpha_k}, u_i^{\alpha_k}]$ and evaluate the fraction of targets lying inside these intervals.
Lower RCE values indicate better overall calibration across the considered confidence levels. For reference, an RCE below 0.02--0.03 is typically considered well-calibrated, values around 0.05 indicate reasonable calibration, and values above 0.1 suggest poor calibration and under- or over-confident predictive intervals.

\subsection{Ablation Study}\label{app:ablation}
In Table~\ref{table:uci-hessian}, we compare several Hessian approximations for LA-TNKM, including the generalized Gauss–Newton (GGN), block-diagonal (Block), and last-core (Last) approximations, across UCI regression datasets. The results show that using the last-core Hessian consistently leads to improved data-fitting performance (lower RMSE) and better uncertainty estimation (lower NLL and RCE).

The superior performance of the last-core approximation suggests that concentrating curvature information on a subset of parameters can effectively capture the essential posterior structure, while reducing approximation error and numerical instability associated with more global or factorized Hessian estimates. This observation aligns with the findings of~\citet{BNN_Fully_Stochastic_2023}, which demonstrate that modeling uncertainty in only a subset of layers is often enough to obtain well-calibrated and competitive Bayesian predictions.

\begin{table*}[!t]
\caption{Average test RMSE, NLL, and RCE on several UCI regression datasets. Models are trained on 90\% of the data and tested on the remaining 10\%. Each experiment is repeated 10 times, and we report mean ± standard deviation for three Hessian matrix approximations for LA-TNKM: GNN, Block, and Last. }
\label{table:uci-hessian}
\begin{center}
\scalebox{0.75}{
\begin{tabular}{||l||c|c|c||c|c|c||c|c|c||}
\toprule
 & \multicolumn{3}{c||}{RMSE $\downarrow$} & \multicolumn{3}{c||}{NLL $\downarrow$} & \multicolumn{3}{c||}{RCE $\downarrow$} \\
 \midrule
 & GNN & Block & Last & GNN & Block & Last & GNN & Block & Last \\
Dataset &  &  &  &  &  &  &  &  &  \\
\midrule
BOSTON & 0.70±0.17 & 0.70±0.17 & \textbf{0.63±0.11} & 0.97±0.12 & 0.99±0.12 & \textbf{0.95±0.16} & 0.09±0.03 & 0.11±0.03 & \textbf{0.05±0.02} \\
CONCRETE & 0.63±0.04 & 0.64±0.04 & \textbf{0.55±0.05} & 0.98±0.06 & 1.01±0.05 & \textbf{0.82±0.09} & 0.06±0.02 & 0.07±0.02 & \textbf{0.03±0.02} \\
ENERGY & 0.44±0.05 & 0.44±0.05 & \textbf{0.05±0.01} & 0.62±0.09 & 0.64±0.07 & \textbf{-1.40±0.04} & \textbf{0.06±0.02} & 0.08±0.03 & 0.20±0.02 \\
KIN8NM & 0.58±0.02 & 0.56±0.05 & \textbf{0.39±0.01} & 0.88±0.02 & 0.86±0.09 & \textbf{0.48±0.02} & 0.04±0.01 & 0.05±0.01 & \textbf{0.02±0.01} \\
NAVAL & \textbf{0.01±0.00} & 0.11±0.03 & 0.08±0.04 & \textbf{-1.96±0.05} & -0.61±0.18 & -1.16±0.64 & 0.44±0.03 & 0.26±0.03 & \textbf{0.23±0.03} \\
POWER & 0.31±0.01 & 0.31±0.01 & \textbf{0.24±0.01} & 0.24±0.03 & 0.24±0.02 & \textbf{-0.01±0.03} & \textbf{0.03±0.01} & \textbf{0.03±0.01} & \textbf{0.03±0.01} \\
PROTEIN & 0.72±0.01 & 0.72±0.01 & \textbf{0.70±0.01} & 1.08±0.01 & 1.09±0.01 & \textbf{1.06±0.01} & 0.05±0.01 & 0.06±0.01 & \textbf{0.04±0.01} \\
RED WINE & 0.86±0.07 & 0.86±0.07 & \textbf{0.84±0.06} & 1.28±0.06 & 1.29±0.06 & \textbf{1.24±0.06} & 0.08±0.01 & 0.08±0.01 & \textbf{0.04±0.02} \\
YACHT & 0.24±0.15 & 0.13±0.12 & \textbf{0.13±0.03} & 0.02±0.36 & 0.07±0.18 & \textbf{-0.52±0.17} & 0.20±0.06 & 0.35±0.09 & \textbf{0.13±0.06} \\
\bottomrule
\end{tabular}
}
\end{center}
\end{table*}

\subsection{Uncertainty on Synthetic Data}\label{app:synthetic}
In Table~\ref{table:x-cube-comparison}, which complements the results shown in Figure~\ref{fig:x3}, we report RMSE, NLL, ECP-95, WCPI-95, and RCE metrics on test for several probabilistic models: LA-BNN, SP-BTN, GP, LA-TNKM, and mean-field Bayesian tensor network (MF-BTN)~\citep{SP_BTN_Mandic_2022}. The results further indicate that LA-TNKM most closely matches the behavior of Gaussian Processes, exhibiting fewer signs of miscalibration (mismatch between nominal and empirical uncertainty) and model misspecification (systematic deviations from the true data-generating process) compared to the competing approaches.

\begin{table*}[!t]
\caption{Comparison of predictive performance and uncertainty calibration on the cubic regression task ($y = x^3 + \epsilon $, where $\epsilon \sim \mathcal{N}(0, 3^2)$) for LA-BNN, MF-BTN, SP-BTN, GP, and the proposed LA-TNKM. Models are trained on 20 data points and evaluated on 100 test points. Metrics reported include root mean squared error (RMSE), negative log-likelihood (NLL), empirical coverage at 95\% (ECP-95), width of the 95\% prediction interval (WCPI-95), and regression calibration error (RCE). Among the evaluated models, LA-TNKM exhibits predictive behavior more consistent with Gaussian process regression, whereas the other baselines show signs of miscalibration or misspecification.}
\label{table:x-cube-comparison}
\begin{center}
\begin{tabular}{||l||c|c|c|c|c||}
\toprule
MODEL & RMSE $\downarrow$ & NLL $\downarrow$ & ECP-95 & WCPI-95 $\downarrow$ & RCE $\downarrow$ \\
\midrule
LA-BNN & 16.727 & 4.226 & 0.750 & 17.447 & 0.112 \\
MF-BTN & 4.883 & 3.730 & 1.000 & 99.559 & 0.289 \\
SP-BTN & 6.202 & 3.257 & 0.970 & 40.168 & 0.097 \\
GP & 3.231 & 2.610 & 0.940 & 13.496 & 0.031 \\
LA-TNKM & 4.097 & 2.802 & 0.960 & 17.202 & 0.032 \\
\bottomrule
\end{tabular}
\end{center}
\end{table*}

\subsection{UCI Regression}\label{app:uci-reg}
Tables~\ref{table:uci-rmse-rce} and~\ref{table:uci-ecp-wcpi} report the average test RMSE, ECP-95, WCPI-95, and RCE for four methods: LA-BNN, MF-BTN, SP-BTN, and the proposed LA-TNKM, while Table~\ref{table:uci-comparison} summarizes the corresponding NLL results. Overall, the proposed LA-TNKM model demonstrates consistently superior uncertainty calibration, outperforming the competing approaches in terms of ECP-95 and RCE across most datasets. This indicates that LA-TNKM produces predictive uncertainties that more accurately reflect the true uncertainty in the data. Furthermore, the WCPI-95 metric demonstrates that LA-TNKM achieves lower predictive interval widths, indicating sharper uncertainty estimates. This suggests that the model balances calibration and precision, avoiding unnecessarily conservative predictions while retaining reliable coverage compared to MF-BTN and SP-BTN baselines.

Although LA-BNN attains lower RMSE and strong pointwise accuracy, it shows weaker uncertainty calibration, with higher RCE and less reliable credible interval coverage. This suggests that, in this setting, its uncertainty estimates may be less suitable for probabilistic modeling.

\begin{table*}[!t]
\caption{Average test RMSE and RCE on several UCI regression datasets. We train on random 90\% of the data and predict on 10\%. Each experiment is repeated 10 times, and we report the mean ± standard deviation. }
\label{table:uci-rmse-rce}
\begin{center}
\scalebox{0.85}{
\begin{tabular}{||l||c|c|c|c||c|c|c|c||}
\toprule
 & \multicolumn{4}{c||}{RMSE $\downarrow$} & \multicolumn{4}{c||}{RCE $\downarrow$} \\
 \midrule
 & LA-BNN & MF-BTN & SP-BTN & LA-TNKM & LA-BNN & MF-BTN & SP-BTN & LA-TNKM \\
Dataset &  &  &  &  &  &  &  &  \\
\midrule
BOSTON & \textbf{0.20±0.06} & 1.02±0.11 & 0.77±0.13 & 0.63±0.11 & 0.46±0.02 & 0.05±0.03 & 0.13±0.04 & \textbf{0.05±0.02} \\
CONCRETE & \textbf{0.21±0.05} & 0.99±0.08 & 0.47±0.05 & 0.55±0.05 & 0.43±0.02 & \textbf{0.03±0.01} & 0.10±0.03 & 0.03±0.02 \\
ENERGY & \textbf{0.04±0.01} & 0.98±0.04 & 0.43±0.05 & 0.05±0.01 & 0.47±0.00 & 0.10±0.01 & \textbf{0.07±0.03} & 0.20±0.02 \\
KIN8NM & \textbf{0.25±0.01} & 1.00±0.01 & 0.57±0.02 & 0.39±0.01 & 0.39±0.01 & \textbf{0.02±0.01} & 0.07±0.01 & \textbf{0.02±0.01} \\
NAVAL & \textbf{0.04±0.03} & 1.00±0.01 & 0.87±0.04 & 0.08±0.04 & 0.50±0.00 & 0.08±0.01 & \textbf{0.04±0.02} & 0.23±0.03 \\
POWER & \textbf{0.22±0.01} & 1.00±0.02 & 0.41±0.01 & 0.24±0.01 & 0.37±0.01 & 0.06±0.01 & 0.09±0.01 & \textbf{0.03±0.01} \\
PROTEIN & \textbf{0.66±0.02} & 1.00±0.01 & 0.92±0.01 & 0.70±0.01 & \textbf{0.03±0.00} & 0.10±0.00 & 0.05±0.00 & 0.04±0.01 \\
RED WINE & \textbf{0.61±0.13} & 0.93±0.06 & 0.82±0.06 & 0.84±0.06 & 0.40±0.04 & 0.10±0.01 & 0.07±0.01 & \textbf{0.04±0.02} \\
YACHT & \textbf{0.04±0.06} & 0.95±0.11 & 0.40±0.23 & 0.13±0.03 & 0.47±0.02 & \textbf{0.11±0.02} & 0.20±0.11 & 0.13±0.06 \\
\bottomrule
\end{tabular}
}
\end{center}
\end{table*}

\begin{table*}[!t]
\caption{Average test ECP-95 and WCPI-95 on several UCI regression datasets. We train on random 90\% of the data and predict on 10\%. Each experiment is repeated 10 times, and we report the mean ± standard deviation. 'N/A' refers to cases where reporting a value is not feasible due to prohibitively large computational demands or lack of numerical stability. }
\label{table:uci-ecp-wcpi}
\begin{center}
\scalebox{0.85}{
\begin{tabular}{||l||c|c|c|c||c|c|c|c||}
\toprule
 & \multicolumn{4}{c||}{ECP-95} & \multicolumn{4}{c||}{WCPI-95 $\downarrow$} \\
 \midrule
 & LA-BNN & MF-BTN & SP-BTN & LA-TNKM & LA-BNN & MF-BTN & SP-BTN & LA-TNKM \\
Dataset &  &  &  &  &  &  &  &  \\
\midrule
BOSTON & 1.00±0.00 & \textbf{0.93±0.03} & 0.98±0.02 & 0.90±0.04 & 14.99±2.71 & 3.90±0.03 & 3.83±0.13 & \textbf{2.07±0.40} \\
CONCRETE & 1.00±0.00 & \textbf{0.95±0.02} & 0.98±0.02 & \textbf{0.95±0.02} & 7.65±0.50 & 3.92±0.02 & 2.41±0.03 & \textbf{2.22±0.20} \\
ENERGY & 1.00±0.00 & \textbf{0.98±0.02} & \textbf{0.98±0.02} & 0.99±0.01 & 5.77±0.15 & 3.94±0.03 & 2.09±0.04 & \textbf{0.32±0.00} \\
KIN8NM & 1.00±0.00 & 0.96±0.00 & 0.97±0.01 & \textbf{0.95±0.01} & 5.15±0.09 & 3.89±0.01 & 2.68±0.00 & \textbf{1.55±0.02} \\
NAVAL & N/A±0.00 & 1.00±0.00 & 1.00±0.00 & \textbf{0.98±0.02} & N/A±0.00 & 3.86±0.01 & 4.27±0.07 & \textbf{0.43±0.19} \\
POWER & 1.00±0.00 & 0.99±0.00 & 0.98±0.00 & \textbf{0.96±0.01} & 3.61±0.08 & 4.12±0.10 & 2.02±0.01 & \textbf{0.97±0.01} \\
PROTEIN & 0.92±0.01 & 0.98±0.00 & 0.98±0.00 & \textbf{0.94±0.00} & \textbf{2.41±0.02} & 3.96±0.00 & 4.12±0.02 & 2.74±0.01 \\
RED WINE & 1.00±0.00 & 0.97±0.02 & 0.96±0.02 & \textbf{0.95±0.02} & 13.09±1.83 & 3.88±0.02 & 3.76±0.04 & \textbf{3.38±0.14} \\
YACHT & 1.00±0.00 & 0.93±0.02 & 0.98±0.03 & \textbf{0.97±0.03} & 5.94±0.12 & 3.92±0.04 & 2.24±0.43 & \textbf{0.70±0.05} \\
\bottomrule
\end{tabular}
}
\end{center}
\end{table*}

\end{document}